\documentclass[11pt]{article}
\usepackage{amsmath}
\usepackage{algorithm}
\usepackage{algpseudocode}

\usepackage[english]{babel}

\usepackage[letterpaper,top=2cm,bottom=2cm,left=3cm,right=3cm,marginparwidth=1.75cm]{geometry}

\usepackage[utf8]{inputenc}
\usepackage{amsmath}
\usepackage{amssymb}
\usepackage{amsthm}
\usepackage{multicol}
\usepackage{multirow}
\usepackage{amssymb}
\usepackage{algpseudocode} 
\usepackage{algorithm}
\usepackage[mathcal]{euscript}
\usepackage[english]{babel}
\usepackage{adjustbox}
\usepackage[T1]{fontenc}
\usepackage{url}            
\usepackage{booktabs}       
\usepackage{amsfonts}       
\usepackage{nicefrac}       
\usepackage{microtype}      
\usepackage{tikz}
\usetikzlibrary{positioning}
\usepackage{graphicx}
\usepackage[nottoc,notlot,notlof]{tocbibind}
\usepackage{caption}
\usepackage{subcaption}
\usepackage{mathabx}
\usepackage[inline]{enumitem}
\usepackage{mathtools}

\newtheorem{proposition}{Proposition}

\newtheorem{remark}{Remark}

\newtheorem{definition}{Definition}

\newcommand{\fbb}{f_{\texttt{BB}}}

\newcommand{\calP}{\mathcal{P}}

\newcommand{\calX}{\mathcal{X}}
\newcommand{\calA}{\mathcal{A}}

\title{\textsc{I-SAFE}: Wasserstein Coherence Metrics for Structural Auditing of Scientific AI Models}

\author{
    Barbara Tarantino\thanks{Department of Economics, University of Pavia, Via S. Felice Al Monastero, 5, Pavia, 27100, Italy, email: barbara.tarantino@unipv.it}
    \and
    Gennaro Auricchio\thanks{Department of Mathematics, University of Padua, Via Trieste, 63, Padua, 35131, Italy, email: gennaro.auricchio@unipd.it}
    \and
    Paolo Giudici\thanks{Department of Economics, University of Pavia, Via S. Felice Al Monastero, 5, Pavia, 27100, Italy, email: paolo.giudici@unipv.it}
}

\begin{document}

\maketitle

\begin{abstract}
Deep learning models are increasingly used in scientific prediction tasks where strong benchmark performance is often interpreted as evidence of scientifically meaningful behavior. This interpretation is fragile, as models may exploit shortcut features, dataset-specific regularities, or distributional biases that are predictive on held-out data but not aligned with domain-relevant structure.
To address this limitation, we introduce the \textsc{I-SAFE} (Interventional Secure, Accurate, Fair and Explainable) framework, a post-hoc distributional auditing framework for scientific AI models centered on the Wasserstein Coherence Metric (WCM). Given a trained black-box predictor and an external structural prior encoding domain knowledge about task-relevant input structure, \textsc{I-SAFE} evaluates raw model outputs under structurally guided perturbations of the input.
The proposed audit measures output-distribution coherence through three complementary metrics: a Quantile-Based Metric (QBM) for location-level coherence, the WCM for ordinal coherence, and a translation-invariant
WCM variant for shape coherence. 
We instantiate \textsc{I-SAFE} on drug--target interaction (DTI) prediction using the Davis
kinase benchmark, KLIFS (Kinase--Ligand Interaction Fingerprints and Structures)
binding-pocket annotations, and three sequence-based DTI models: DeepConvDTI,
DeepDTA, and TAPB.
Although the models operate in a comparable predictive regime, \textsc{I-SAFE} reveals substantially different distributional response profiles, a distinction invisible to accuracy-based evaluation. The framework is model-agnostic and applicable to any domain where inputs admit a structured decomposition and an external prior is available.
\end{abstract}

\section{Introduction}
\label{sec:intro}

    %

%
Held-out predictive performance remains the dominant criterion for evaluating deep learning models in scientific machine learning, including molecular property prediction, drug--target interaction~(DTI), and gene perturbation response~\cite{wu2018moleculenet,huang2021therapeutics,lotfollahi2023predicting}. 
In these settings, strong benchmark
performance is often interpreted as evidence of scientifically meaningful behavior.
%
This interpretation is fragile, as predictive accuracy measures are affected by statistical regularities that are present in the benchmark distribution, and not only on whether the model rightly identified the relevant structures of the input~\cite{pearl2009causality}.
%
%
This gap is central to modern machine learning, where high-capacity models can achieve strong performance by exploiting shortcuts that do not necessarily align with the scientific structure of interest~\cite{geirhos2020shortcut,ilyas2019adversarial}.
In scientific applications, this limitation is empirical as well as conceptual. 
For example, high-performing
protein--ligand affinity models have been shown to rely substantially on ligand memorisation rather than interaction-specific information~\cite{mastropietro2023learning}.
%
%
Specifically, in DTI, apparent progress on standard benchmarks can be influenced by target bias, scaffold effects, leakage, and other forms of distributional bias~\cite{lin2025tapb,wallach2018ligand,graber2025resolving}, which leads the model to fail on structurally similar compounds with markedly different functional behaviour~\cite{vantilborg2022activity}.
%

%
These findings expose a mismatch between the predictive success of a model and its ability to capture the structural properties of the problem.
Standard benchmark evaluation can establish that a model predicts well on a given distribution, but cannot characterize how model outputs reorganize when scientifically relevant input structure is perturbed.
%

%
%
%
A principled response to this problem is to move from observational to interventional evaluation, where the question is not only whether a model predicts correctly, but how its predictions change under controlled perturbations of the input.
In causal terms, mechanistic claims concern responses to interventions rather than associations observed under a fixed distribution~\cite{pearl2009causality,peters2016causal}. 
This perspective has informed several post-hoc lines of work, including causal abstraction of internal representations~\cite{geiger2021causal,geiger2024finding}, model-randomisation tests for explanation faithfulness~\cite{adebayo2018sanity}, and perturbation probes of reasoning behaviour in language models~\cite{xu2025reimagine,zhang2025llmscan}. 
In scientific prediction, domain knowledge provides a structural prior over inputs, enabling comparison of model responses to perturbations on prior-selected versus non-prior selected components. 
%
%
%
%
This yields a post-hoc, prior-relative audit of whether the model response
is organized with respect to meaningful input structure, as formalized in~\cite{tarantino2026isaacauditingcausalreasoning}.
%

%

\paragraph{Our Contribution.}

In this paper, we introduce \textsc{I-SAFE} (Interventional Secure, Accurate, Fair and Explainable), a post-hoc distributional auditing framework for trained black-box scientific predictors. 
While existing intervention-based audits rely on scalar summaries, the \textsc{I-SAFE}
framework evaluates the full output distribution induced by the model, capturing how
output ranks reorder under prior-guided perturbations of selected input components.
%
Our main theoretical contribution thus consists in defining a set of metrics that capture different levels of ranking coherence: the Quantile-Based Metric for location-level coherence, the Wasserstein Coherence Metric for ordinal coherence, and a translation-invariant Wasserstein Metric for distributional shape coherence. 
For each metric, the prior-relative contrast compares the ranking-coherence under outside-prior controls and prior-selected perturbations. 
Positive contrasts indicate that perturbations of prior-selected components induce more coherent output responses than their controls outside the prior.
%
%

%
We then instantiate \textsc{I-SAFE} on the Davis kinase benchmark~\cite{davis2011comprehensive}, auditing three sequence-based DTI models, DeepDTA~\cite{ozturk2018deepdta}, DeepConvDTI~\cite{lee2019deepconv}, and TAPB~\cite{lin2025tapb}, using Kinase--Ligand Interaction Fingerprints and Structures (KLIFS) binding-pocket annotations as an external structural prior~\cite{kooistra2016klifs,kanev2021klifs}.
While the audited models operate in a comparable predictive regime, 
%
%
we show that their interventional behaviour differs substantially. 
In particular, TAPB is the only model for which KLIFS-aligned pocket perturbations induce significantly more coherent quantile-level and ordinal responses than non-pocket controls.
%
Our results show that predictive performance and distributional coherence capture distinct aspects of scientific model behaviour, indicating that \textsc{I-SAFE} provides a statistically grounded, prior-relative test of how trained scientific predictors operate under structurally meaningful perturbations.
We stress that the audit does not establish causal validity of the model, nor does it require access to the model itself, it rather provides reusable, model-agnostic metrics for evaluating black-box model responses under structurally guided perturbations, applicable to any setting where inputs admit a structured decomposition and a domain prior is available.
%

\section{Related Work}
\label{sec:related}


\paragraph{From explanation to interventional model analysis.}
Attribution and explanation methods provide tools for inspecting trained models. 
Saliency maps, local surrogate models, Shapley-based explanations, and integrated gradients~\cite{ribeiro2016trust,lundberg2017unified,sundararajan2017axiomatic,simonyan2014deep} identify input features that influence predictions, supporting transparency and debugging~\cite{doshi2017rigorous}.
However, such explanations are mostly associational: they indicate the influential features under the observed input distribution, without establishing whether predictions depend on scientifically relevant structure.
Their fragility has been documented empirically: Adebayo et al.~\cite{adebayo2018sanity} show that saliency methods can be invariant to model parameter randomisation, calling into question their mechanistic faithfulness.
Intervention-based approaches address this limitation by asking how the model changes under controlled modifications. 
Causal abstraction and related
methods~\cite{geiger2021causal,geiger2024finding} compare neural representations with interpretable causal variables through interchange interventions on internal activations. 
These methods provide a formal route to mechanistic comparison, but usually require access to internal representations and a target causal structure. 
Input-level perturbation methods avoid this requirement by probing the input--output matching induced by the model and have been applied to language models~\cite{xu2025reimagine,zhang2025llmscan}. 
%
%
In scientific prediction, post-hoc structural auditing under an external prior has been proposed to contrast mechanistic and spurious perturbations through scalar summaries~\cite{tarantino2026isaacauditingcausalreasoning}.
\textsc{I-SAFE} retains the post-hoc black-box setting, but shifts
the object of comparison from scalar aggregated values to the distributional changes in the model output.

\paragraph{Shortcut learning and structural robustness.}
The need for auditing frameworks is reinforced by extensive evidence that predictive success can be driven by shortcuts or unstable correlations rather than task-relevant structures~\cite{geirhos2020shortcut,ilyas2019adversarial}.
In molecular prediction and DTI, this issue appears as ligand memorisation~\cite{mastropietro2023learning}, target prior bias~\cite{lin2025tapb}, benchmark artifacts related to leakage and split design~\cite{wallach2018ligand,graber2025resolving}, and failures on activity cliffs, where structurally similar compounds have different functional effects~\cite{vantilborg2022activity}. 
These findings motivate 
the search for predictors relying on stable relations. 
Invariant causal prediction~\cite{peters2016causal} and Invariant Risk Minimization~\cite{arjovsky2019invariant}, formalize stability across environments as a route toward more robust prediction.
The objective of \textsc{I-SAFE} is complementary. Rather than modifying the learning procedure or requiring environment annotations, it asks whether a trained model responds coherently when input components identified by external scientific knowledge are perturbed. 
This distinction is important in scientific AI, where models are used as black-box predictors and retraining is impractical, unavailable, or insufficient to diagnose the predictor.
A parallel limitation affects intervention-based evaluation itself, where the richness of the auditing framework is bounded by the statistical resolution at which model responses are characterized.

\paragraph{Beyond aggregate performance.}
Single aggregate metrics are often too coarse to characterize behaviour relevant for downstream use. 
In language model evaluation, DecodingTrust~\cite{wang2023decodingtrust}, HELM~\cite{liang2023helm}, and
BIG-bench~\cite{srivastava2023beyond} address this issue by organizing assessment across multiple capabilities, risks, and scenarios. 
Likewise, in scientific model evaluation, held-out performance establishes predictive adequacy on a benchmark distribution, but not how predictions reorganize under perturbations of scientifically meaningful input structure.
A similar limitation arises within intervention-based evaluation itself. 
Model responses are often reduced to average effects, scalar sensitivity scores, or a small set of moments~\cite{geiger2021causal,zhang2025llmscan,tarantino2026isaacauditingcausalreasoning}. 
These summaries detect location-level contrasts, but they might miss whether output distributions shift coherently, preserve ordinal structure, or change shape.
A standard approach to overcome single aggregate metrics relies on using two-sample distributional comparison~\cite{cramer1928composition,anderson1962distribution} and optimal transport theory~\cite{villani2009optimal,peyre2019computational}, as these metrics capture the geometry of the underlying space. 
%
\textsc{I-SAFE} brings this perspective to structural auditing by decomposing the distributional response to perturbations into three complementary axes: location, ordinal structure, and shape. Each axis captures aspects of output reorganization invisible to scalar summaries.

\section{The I-SAFE Framework}
\label{sec:framework}

%
%
%
%


%
In this section, we formalize the \textsc{I-SAFE} framework as a post-hoc auditing procedure for fixed black-box predictors. 
\textsc{I-SAFE} leverages prior-guided perturbations to induce paired profiles of raw model outputs and evaluates the coherence of their distributional reorganization after the intervention. 
%
%
We consider a predictor \(\fbb:\calX\to\mathbb{R}\) accessed only through its input--output map, where \(\calX=\prod_{m=1}^{M}\calX_m\) admits a decomposition in \(M\) identifiable components. 
For an input \(x=(x^{(1)},\ldots,x^{(M)})\), the components of $x$ define the units on which interventions act. 
Throughout the audit, \(\fbb(x)\) denotes the raw model output on input $x\in\calX$, before thresholding, calibration, or downstream decision rules.

\subsection{Problem formulation}
\label{sec:formulation}

%
%
 
%
%
%

%
%
%

%
Given a black-box predictor $\fbb$, the audit is performed on $\calA = \{x_i\}_{i=1}^{N} \subset\calX$ which is disjoint from the data used to train $\fbb$. 
The set $\calA$ contains the elements of $\calX$ on which all perturbations are applied and all output distributions are compared.
%
%
%
To determine how to perform a perturbation, we have access to a structural prior over $\calX$, derived from domain knowledge that does not depend on $\fbb$.

\begin{definition}[Structural prior]
\label{def:prior}
We say that a map $\calP : \calX \to 2^{[M]}$ is a \emph{structural prior} if it satisfies $\emptyset \subsetneq \calP(x) \subsetneq [M]$ for all $x \in \calX$.
\end{definition}

For each input $x$, the set $\calP(x) \subset [M]$ indexes the components that domain knowledge regards as relevant for the prediction task.
The strict inclusions ensure we avoid cases in which no components are relevant (i.e. $\calP(x) = \emptyset$) and the case in which all components are relevant (i.e. $\calP(x)\! =\! [M]$).

\begin{remark}[Epistemic status of \(\calP\)]
\(\calP\) encodes structural information derived from independent domain evidence, not ground-truth causal structure~\cite{pearl2009causality}. Accordingly, the \textsc{I-SAFE} framework quantifies prior-relative structural alignment of model responses rather than causal validity in an absolute sense. 
The strength of the audit therefore depends on the scientific quality of the information encoded in \(\calP\), which is consistent with the principles of veridical data science~\cite{yu2020veridical}.
\end{remark}
%


\subsection{Prior-relative intervention design}
\label{sec:primitive}

%
Given $\calA$ and a structural prior $\calP$, the audit compares two classes of interventions defined relative to $\calP$: the interventions acting on $\calP(x)$, and the interventions acting on the complement of $\calP(x)$.
The comparison is thus prior-dependent and evaluates whether the model responds differently when the same perturbation is applied to prior-selected components rather than to components outside the prior.

%

\begin{definition}[Mechanistic and spurious perturbations]
\label{def:structured_spurious}
A perturbation \(\varphi_{\calP}:\calX\to\calX\) is \emph{mechanistic} if it
acts only on components selected by the prior, that is
\begin{equation}
    \bigl(\varphi_{\calP}(x)\bigr)_i = x_i
    \qquad \forall\, i \notin \calP(x).
\end{equation}
Likewise, $\varphi$ is said to be \emph{spurious} if it leaves all prior-selected components unchanged, that is
\begin{equation}
    \bigl(\varphi_{\calP}(x)\bigr)_i = x_i
    \qquad \forall\, i \in \calP(x).
\end{equation}
%

%

\end{definition}

The term \emph{spurious} is used only relative to the specified prior \(\calP(x)\).
Notice that this does not imply that such components are irrelevant in an absolute scientific sense.
To make the mechanistic-spurious comparison interpretable, we distinguish the perturbation rule from the intervention support.

\begin{definition}[Perturbation operator]
\label{def:operator}
A \emph{perturbation operator} is a map $\Phi:\calX\times 2^{[M]}\to\calX$ such that, for every \(x\in\calX\) and \(S\subseteq[M]\),
\(\Phi(x,S)\) differs from \(x\) only on components indexed by \(S\).
\end{definition}

Given $x\in\calX$, the mechanistic support of $x$ is \(\calP(x)\).
A spurious support is a set \(S^{\mathrm{spur}}(x)\subseteq [M]\setminus\calP(x)\). 
Both supports are modified using the same operator \(\Phi\).
Thus, the two interventions share the perturbation rule and differ only in the position of their support relative to the prior.

\begin{definition}[Matched perturbation pair]
\label{def:matched_pair}
A pair \((\varphi^{\mathrm{mech}},\varphi^{\mathrm{spur}})\) is a
\emph{matched perturbation pair} on \(\calA\) if there exists a perturbation
operator \(\Phi\) such that, for every \(x\in\calA\),
\begin{equation}
    \varphi^{\mathrm{mech}}(x)=\Phi(x,\calP(x)),
    \qquad
    \varphi^{\mathrm{spur}}(x)=\Phi(x,S^{\mathrm{spur}}(x)),
\end{equation}
where $S^{\mathrm{spur}}(x)\subseteq [M]\setminus\calP(x)$ and $|S^{\mathrm{spur}}(x)|=|\calP(x)|$.
\end{definition}

A matched perturbation pair controls the two intervention degrees of the method: the perturbation rule and the number of perturbed components. 
Acting on the same number of entries enables the isolation of the prior alignment role in the model response, according to the prior $\calP$.
%

\begin{definition}[Interventional response profile]
\label{def:response_profile}
Let \(\varphi:\calX\to\calX\) be either a mechanistic or a spurious
perturbation, and let \(\calA=\{x_i\}_{i=1}^{N}\subset\calX\) be the auditing
set. 
Then, we define 
\begin{equation}
    \label{eq:output_vectors}
    V_{\calA} = \bigl(\fbb(x_i)\bigr)_{i=1}^{N}
    \qquad\qquad\qquad \text{and} \qquad\qquad\qquad
    V_{\varphi(\calA)} = \bigl(\fbb(\varphi(x_i))\bigr)_{i=1}^{N}.
\end{equation}
The pair $\mathcal{R}_{\fbb}(\varphi;\calA) = (V_{\calA},V_{\varphi(\calA)})$
is the \emph{interventional response profile} of \(\fbb\) under \(\varphi\) on
\(\calA\).
\end{definition}

The response profile contains the raw model outputs before and after intervention on the same auditing population, preserving both the empirical output distributions and the input-wise pairing induced by the perturbation.
For each matched pair, \textsc{I-SAFE} compares the mechanistic and spurious response profiles through the coherence metrics introduced in the following section.

\subsection{The I-SAFE Auditing Metrics}
\label{sec:metrics}

%
%
%
%

%
%
In this section, we introduce a set of three metrics to quantify how the raw output distribution of $\fbb$ changes in terms of relative ranking coherence under a perturbation $\varphi$. 
In line with Secure, Accurate, Fair and Explainable (SAFE) metrics~\cite{auricchio2026rank,babaei2025rank}, we define our metrics to represent the percentage of coherence explainable by input differences.
In other words, the metrics should attain values close to $1$ when the coherence is small and values close to $0$ when the coherence is high.
%
%
%
%

%
First, we assess the location-level distributional change by comparing empirical quantiles of the original and perturbed output profiles, \textit{i.e.},
\(\{\fbb(x)\}_{x\in\calA}\) and
\(\{\fbb(\varphi(x))\}_{x\in\calA}\).
%
%
%
%
Given a vector of target quantile levels $q=(q_1,\ldots,q_K)$, we denote by \(q^{(\calA)}_k\) the empirical \(q_k\)-quantile of \(\{\fbb(x_i)\}_{i=1}^{N}\), and by \(q^{(\varphi(\calA))}_k\) ithe empirical \(q_k\)-quantile of \(\{\fbb(\varphi(x_i))\}_{i=1}^{N}\).
Finally, we set $q^{(\calA)}=(q^{(\calA)}_1,\ldots,q^{(\calA)}_K)$  and by $q^{(\varphi(\calA))}=(q^{(\varphi(\calA))}_1,\ldots,q^{(\varphi(\calA))}_K)$. 
%
%
%
%
The \textit{Quantile-based Metric} (QBM) compares the displacement of these representative locations with the total Mean Square Error of the values.


\begin{definition}
Given an auditing set $\calA$, a perturbation $\varphi$, a black-box model $\fbb$, and a quantile vector \(q=(q_1,\ldots,q_K)\), the Quantile-Based Metric (QBM) is defined as
    \begin{equation}
        QBM(\varphi;\fbb)=\Bigg(1-\sqrt{\frac{\frac{1}{K}\sum_{i=1}^K |q^{\varphi(\calA)}_i-q^{(\calA)}_i|^2}{\frac{1}{|\calA|}\sum_{i=1}^{|\calA|}|\fbb(\varphi(x_i))-\fbb(x_i)|^2}}\Bigg)_+,
    \end{equation}
\end{definition}
where $(x)_+$ denotes the positive part of $x$.
%

%
%
%


%
%
%
By definition, the lower the values of QBM, the larger the shift in the quantile displacement.
Conversely, larger values of QBM indicate that the quantile structure remains mostly unaffected by the perturbation despite potentially large pointwise changes.

\begin{proposition}
\label{prop0}
    For every auditing set $\calA$, perturbation $\varphi$, black-box model $\fbb$, and quantile vector $q$ it holds $QBM(\varphi;\fbb)\in[0,1]$.
\end{proposition}

Notice that QBM depends on both the output distributions and the number of quantiles we consider. 
Coarser quantile grids summarize location-level changes, whereas finer grids retain more information about the empirical distributions of \(\{\fbb(x)\}_{x\in\calA}\) and \(\{\fbb(\varphi(x))\}_{x\in\calA}\). 
When the grid is refined to the empirical order statistics, the QBM becomes a normalized optimal-transport metric, which we name \textit{Wasserstein Coherence Metric} (WCM) and does not depend on the quantile choice.
%


%
%

%

\begin{definition}
\label{def:WCM}
    Given an auditing set $\calA$, a perturbation $\varphi$, and a black-box model $\fbb$, we define the Wasserstein Coherence Metric (WCM) as follows
    \begin{equation}
    \label{def:WCM_defo}
        WCM(\varphi;\fbb):=1-\sqrt{\frac{\min_{\pi\in \Pi_n(\calA)}\sum_{x\in\calA}|\fbb(\varphi(\pi(x)))-\fbb(x)|^2}{\sum_{x\in\calA}|\fbb(\varphi(x))-\fbb(x)|^2}},
    \end{equation}
    where \(\Pi(\calA)\) denotes the set of all permutations of the auditing inputs in \(\calA\).
\end{definition}

%
%

%

The denominator of \eqref{def:WCM_defo} measures the paired output displacement under the natural
correspondence $x \mapsto \varphi(x)$, while the numerator measures the smallest
displacement attainable after optimally reordering the perturbed outputs. WCM is
therefore small when the natural input-wise pairing is close to this optimal
reordering, and large when the perturbed outputs can be substantially better
matched only after reordering.

%

%
%

Definition \ref{def:WCM} yields two useful invariance properties \begin{enumerate*}[label=(\roman*)]
    \item WCM is normalized and
    \item WCM is invariant under positive rescaling of the model output. 
\end{enumerate*} 
%
%
These properties make WCM an easy to interpret and scale-free measure of ordinal coherence.
%

\begin{proposition}
\label{prop1}
    For every auditing set $\calA$, perturbation $\varphi$, and black-box model $\fbb$, it holds $WCM(\varphi;\fbb)\in[0,1]$.
    Moreover, 
    \begin{enumerate*}[label=(\roman*)]
        \item $WCM(\varphi;\fbb)=0$ if and only if the pairing $\{(\fbb(x),\fbb(\varphi(x)))\}_{x\in\calA}$ is a coupling between the perturbed and original output that maximizes the covariance; and
        \item $WCM(\varphi;\fbb)=1$ if and only if there exists a permutation, namely $\pi$, such that $\pi$ is not the identity and for which it holds $\fbb(x)=\fbb(\varphi(\pi(x)))$ for every $x\in\calA$.
    \end{enumerate*}
    Lastly, the $WCM$ is invariant under change of scales, meaning that its value does not change if we alter unit of measure of the output layer, i.e. $WCM(\varphi;\fbb)=WCM(\varphi;\lambda\fbb)$ for any $\lambda>0$.
\end{proposition}

From Proposition \ref{prop1}, we infer that the WCM measures the covariance alignment between the natural coupling induced by the data and the perturbation, that is $(\fbb(x),\fbb(\varphi(x)))$, and the one that maximises the covariance.
This induces a natural connection between the WCM and the Wasserstein Distance between the empirical distribution induced by the black-box model $\fbb$ output before and after the perturbation.
Consequently, we can express the minimum in the numerator of \eqref{def:WCM_defo} as the euclidean norm of the vectors $(\fbb(\varphi(x)))_{x\in \calA}$ and $(\fbb(x))_{x\in \calA}$ rearranged in non-decreasing order.

\begin{proposition}
\label{prop2}
    Given a perturbation $\varphi$, a black-box model $\fbb$, and an auditing set $\calA$, then
    \begin{equation}
    \label{eq:WCM_to_Wass}
        WCM(\varphi;\fbb):=1-\sqrt{\frac{\sum_{i=1}^{|\calA|}\big((V_{\calA})_{r_i}-(V_{\varphi(\calA)})_{r_i^{(\varphi)}}\big)^2}{\sum_{i=1}^{|\calA|}\big((V_{\calA})_{i}-(V_{\varphi(\calA)})_{i}\big)^2}}%
    \end{equation}
    where 
    %
    \begin{enumerate*}[label=(\roman*)]
        \item $V_{\calA}$ and $V_{\varphi(\calA)}$ are defined as in \eqref{eq:output_vectors} and 
        \item $r_i$ is a monotone non-decreasing reordering of $V_{\calA}$ and $r_i^{(\varphi)}$ is a non-decreasing reordering of $V_{\phi(\calA)}$.
    \end{enumerate*}
\end{proposition}

\begin{remark}
    Owing to Proposition \ref{prop2} and the characterization of the Wasserstein Distances $W_2$~\cite{villani2009optimal}, we have the following.
    If we set $\mu_{V_{\calA}}=\frac{1}{|\calA|}\sum_{x_i\in\calA}\delta_{\fbb(x_i)}$ and $\mu_{V_{\varphi(\calA)}}=\frac{1}{|\calA|}\sum_{x_i\in\calA}\delta_{\fbb(\varphi(x_i))}$ to be the empirical measures induced by $V_{\calA}$ and $V_{\varphi(\calA)}$, respectively, then it holds $W_2(\mu_{V_{\varphi(\calA)}},\mu_{V_{\calA}})=\frac{1}{|\calA|}\sum_{i=1}^{|\calA|}\big((V_{\calA})_{r_i}-(V_{\varphi(\calA)})_{r_i^{(\varphi)}}\big)^2$.
    Within this interpretation, the WCM is the percentage of ranking coherence that cannot be explained by the natural variance of the model.
\end{remark}
Lastly, we introduce the third I-SAFE metric, the Translation-Invariant WCM, which measures only differences in distributional shape.

\begin{definition}
    Given an auditing set $\calA$, a perturbation $\varphi$, and a black-box model $\fbb$, we define the translation invariant Wasserstein Coherence Metric (TI-WCM) as follows
    \begin{equation}
        TI-WCM(\varphi,\fbb):=1-\frac{\sqrt{W_2(\mu_{V_{\varphi(\calA)}},\mu_{V_{\calA}})^2-(m_\calA-m_{\varphi(\calA)})^2}}{\ell_2(V_{\varphi(\calA)},V_{\calA})},
    \end{equation}
    where $m_\calA$ is the mean value of $\{\fbb(x)\}_{x\in\calA}$ and $m_{\varphi(\calA)}$ is the mean value of $\{\fbb(\varphi(x)\}_{x\in\calA}$.
\end{definition}

The TI-WCM is invariant under model bias, meaning that if we add a constant value to the output of the model $\fbb$, the TI-WCM does not change.
This makes the TI-WCM a stricter metric than the WCM introduced in Definition \ref{def:WCM}, as shown in~\cite{auricchio2020equivalence}.
Lastly, we notice that the TI-WCM is still a normalized value between $0$ and $1$ that measures the coherence between the ranking of the two output sets.
The full discussion is deferred to Appendix~\ref{app:tiwcm}.

\section{Experiments}
\label{sec:experiments}

We evaluate \textsc{I-SAFE} on DTI prediction, a scientific task in which
binding-pocket annotations provide a natural structural prior over target residues.
In kinase--inhibitor prediction, these annotations identify protein regions involved
in molecular recognition, while sequence-based DTI models may also exploit global
sequence statistics, target bias, or dataset-specific regularities~\cite{lin2025tapb,wallach2018ligand,graber2025resolving}.
This makes DTI a suitable setting for testing whether models with comparable
predictive performance exhibit different prior-relative interventional behaviour.
The goal is not to rank models by accuracy, but to evaluate how raw output
distributions reorganize under matched perturbations of mechanistic and spurious input components.



\subsection{Task and framework instantiation}
\label{sec:task}

Each input is a drug--target pair \(x=(d,t)\), where \(d\) is a drug represented
by a Simplified Molecular Input Line Entry System (SMILES) string and
\(t=(t_1,\ldots,t_M)\) is a protein target represented as an amino-acid
sequence of length \(M=|t|\). The audited model returns a raw score
\(\fbb(d,t)\in\mathbb{R}\), used directly for the \textsc{I-SAFE} metrics before any
thresholding, calibration, or downstream decision rule. 
Labels $y\in\{0,1\}$ define the underlying interaction task and are used only for predictive-regime verification; the \textsc{I-SAFE} metrics are computed exclusively from raw model scores, without label information at any stage of the audit.
In this instantiation, the identifiable input components are the residues of
the target sequence, so that \(x^{(m)}=t_m\) for \(m=1,\ldots,M\). Interventions
act exclusively on the protein target \(t\), while the drug representation
\(d\) is held fixed. Thus, the audit characterizes target-side interventional
behaviour under a fixed drug context.
 
\paragraph{Structural prior.}
The structural prior \(\calP(t)\) is derived from KLIFS~\cite{kooistra2016klifs,kanev2021klifs}, a curated
resource of kinase--ligand structural annotations. For each kinase target,
\(\calP(t)\subseteq\{1,\ldots,M\}\) denotes the residue indices corresponding
to the annotated binding pocket. This prior is specified independently of the
audited models and is used only to define prior-aligned intervention scopes. In
the auditing set, the median pocket size is 85 residues.

%
%
 
 
\subsection{Dataset and auditing set}
\label{sec:data}

Experiments are conducted on the Davis benchmark~\cite{davis2011comprehensive},
a standard kinase inhibitor dataset for DTI prediction. We use the
train/validation/test splits released with the TAPB benchmark~\cite{lin2025tapb}
and apply them consistently across all audited models. These splits are used for
model training and predictive-regime verification, while interventional auditing
is performed post-hoc on the structurally valid subset of the test split.
The auditing set is obtained by deterministic filtering. We retain targets with
KLIFS binding-pocket annotations and for which annotated pocket residues can be
unambiguously mapped to the amino-acid sequence. This yields \(208\) kinase
targets from the \(379\) test targets and \(3{,}044\) drug--target pairs, each
with a unique residue-level interventional representation. The median KLIFS
pocket size is \(85\) residues, and exact cardinality matching is achieved for
all audited interventions. Full structural coverage statistics are reported in
Appendix Table~\ref{app:tab:coverage}.

For each audited pair \((d,t)\in\calA\) and perturbation operator
\(\varphi\in\Phi\), we construct a matched pair of intervention scopes.
The prior-aligned scope is drawn from the KLIFS binding-pocket residues
\(\calP(t)\), whereas the prior-misaligned scope is sampled from
\(\{1,\ldots,M\}\setminus\calP(t)\) with identical cardinality. The same
operator is applied to both scopes, ensuring that the resulting
intervention classes differ only in their alignment with the structural
prior while preserving the perturbation rule and intervention size. In
our experiments, \(\Phi\) contains masking, which replaces selected
residues with a dedicated mask token, and physicochemically constrained
substitution, which replaces each selected residue with an amino acid
from the same biochemical class.


 
\subsection{Audited models and predictive regime}
\label{sec:models}

We audit three sequence-based DTI models that differ in their target-side inductive biases. \textbf{DeepDTA}~\cite{ozturk2018deepdta} provides a
convolutional baseline over SMILES and amino-acid sequences.
\textbf{DeepConvDTI}~\cite{lee2019deepconv} emphasizes protein-sequence representation through target-side convolutions and global pooling.
\textbf{TAPB}~\cite{lin2025tapb} incorporates target-aware attention together with an interventional debiasing objective for target-prior bias. All models are
audited post-hoc from frozen checkpoints trained with their original protocols, without model-specific tuning or adjustment during the \textsc{I-SAFE} audit.

\paragraph{Predictive regime verification.}
Before auditing, we use the area under the receiver operating characteristic
curve (AUROC) on the Davis auditing subset \(\calA\) only to
verify that the models operate in a comparable predictive regime.
Appendix Table~\ref{app:tab:auroc} reports mean AUROC across five training
seeds, with 95\,\% confidence intervals obtained by non-parametric bootstrap
($B=100$) over $\calA$: DeepConvDTI \(0.876\) \([0.875,0.878]\), TAPB \(0.882\) \([0.851,0.899]\), and DeepDTA \(0.907\) \([0.902,0.913]\). The
largest inter-model difference is approximately three AUROC points. We
therefore treat subsequent \textsc{I-SAFE} contrasts as comparisons of
interventional response within a common predictive regime, rather than as
differences in baseline predictive performance.




\subsection{Interventional response structure}
\label{sec:results}

We analyze the interventional response profiles induced by mechanistic and spurious perturbations. For each metric
\(M\in\{\mathrm{QBM},\mathrm{WCM},\mathrm{TI\text{-}WCM}\}\), we report both class-specific values and the prior-relative contrast \(\Delta M=M_{\mathrm{spurious}}-M_{\mathrm{mechanistic}}\). Since lower values indicate more coherent output responses, positive \(\Delta M\) denotes greater
coherence under mechanistic perturbations than under matched spurious controls; negative values denote the reverse. QBM is computed using \(q=(0.25,0.50,0.75)\), the lower quartile, median, and upper quartile of the output distribution; sensitivity to the quantile grid is reported in Appendix Table~\ref{app:tab:qbm-sensitivity}.

\paragraph{Absolute coherence by perturbation class.}
Table~\ref{tab:absolute} reports the \textsc{I-SAFE} metrics separately for
mechanistic and spurious perturbations. These values are diagnostic
rather than ranking metrics. They identify which aspect of the raw output profile is reorganized by a given intervention, separating changes in distributional location, output ordering,
and residual shape.

Under mechanistic perturbations, \textbf{TAPB} shows the strongest coherence on
the two axes that retain location information, with the lowest QBM
(\(0.515\), \([0.476,0.553]\)) and WCM
(\(0.443\), \([0.419,0.467]\)). Thus, perturbing KLIFS pocket residues induces
an organized change in \textbf{TAPB} raw scores, both in distributional location and output
ranking. \textbf{DeepDTA} shows the least coherent response under the same perturbations,
with the highest QBM (\(0.842\), \([0.784,0.897]\)) and WCM
(\(0.768\), \([0.735,0.799]\)), indicating that predictive adequacy does not
imply coherent behaviour under structurally targeted intervention. \textbf{DeepConvDTI} lies between these regimes, with high QBM (\(0.737\), \([0.678,0.799]\)) but lower WCM (\(0.512\), \([0.488,0.537]\)), suggesting partial preservation of
score ordering without an equally coherent quantile-level shift.

The comparison with spurious perturbations already reveals distinct response
profiles. For \textbf{TAPB}, QBM and WCM are lower under mechanistic than
spurious perturbations, whereas \textbf{DeepConvDTI} shows the reverse pattern
on WCM and \textbf{DeepDTA} changes little across classes. TI-WCM further
qualifies the interpretation: \textbf{TAPB} has high mechanistic TI-WCM
(\(0.785\), \([0.751,0.812]\)) despite favourable QBM and WCM, indicating that
its response is not primarily a residual shape-preserving effect after removing
the mean shift. The absolute metrics therefore localize the main signal to
location and ranking. The contrast analysis below tests whether this structure is
selective for KLIFS-aligned pocket perturbations relative to non-pocket
controls.

\begin{table}[t]
\caption{\textsc{I-SAFE} absolute metric values under mechanistic and spurious
perturbations (operator=all, 5 seeds, 95\,\% CI).}
\label{tab:absolute}
\centering
\small
\begin{tabular}{llcc}
\toprule
Metric & Model & Mechanistic (95\,\% CI) & Spurious (95\,\% CI) \\
\midrule
\multirow{3}{*}{QBM}
  & DeepConvDTI & 0.737 (0.678, 0.799) & 0.726 (0.673, 0.774) \\
  & DeepDTA     & 0.842 (0.784, 0.897) & 0.821 (0.766, 0.873) \\
  & TAPB        & 0.515 (0.476, 0.553) & 0.607 (0.572, 0.647) \\
\midrule
\multirow{3}{*}{WCM}
  & DeepConvDTI & 0.512 (0.488, 0.537) & 0.466 (0.441, 0.489) \\
  & DeepDTA     & 0.768 (0.735, 0.799) & 0.755 (0.724, 0.785) \\
  & TAPB        & 0.443 (0.419, 0.467) & 0.507 (0.480, 0.533) \\
\midrule
\multirow{3}{*}{TI-WCM}
  & DeepConvDTI & 0.521 (0.495, 0.546) & 0.478 (0.453, 0.501) \\
  & DeepDTA     & 0.812 (0.783, 0.841) & 0.798 (0.771, 0.827) \\
  & TAPB        & 0.785 (0.751, 0.812) & 0.747 (0.713, 0.776) \\
\bottomrule
\end{tabular}
\end{table}

\paragraph{Prior-relative coherence contrasts.}
Figure~\ref{fig:delta} (and Table~\ref{app:tab:delta} in Appendix) reports the prior-relative contrasts.
These contrasts test whether perturbing KLIFS
pocket residues induces a more coherent distributional response than applying
the same operator, with the same cardinality, outside the pocket.

\textbf{TAPB} shows the clearest prior-relative pattern on the two metrics that retain location information. It achieves \(\Delta\mathrm{QBM}=+0.093\) (\([0.039,0.147]\)) and \(\Delta\mathrm{WCM}=+0.063\) (\([0.028,0.099]\)). This indicates that its coherent response is selectively
associated with KLIFS-aligned pocket perturbations rather than with generic target-sequence perturbation, and is expressed both through quantile-level displacement and output ordering. The other architectures follow a different pattern. For \textbf{DeepConvDTI}, \(\Delta\mathrm{QBM}=-0.011\) (\([-0.090,0.068]\)) and \(\Delta\mathrm{WCM}=-0.046\) (\([-0.081,-0.012]\)), indicating that the ordinal component is more coherent under non-pocket controls than under pocket
perturbations. For \textbf{DeepDTA}, \(\Delta\mathrm{QBM}=-0.021\) (\([-0.099,0.057]\)) and \(\Delta\mathrm{WCM}=-0.013\) (\([-0.057,0.031]\)), showing no comparable prior-aligned coherence advantage.

The translation-invariant component further refines the interpretation. \(\Delta\mathrm{TI\text{-}WCM}\) is negative across all models:
\(-0.043\) (\([-0.078,-0.009]\)) for \textbf{DeepConvDTI},
\(-0.014\) (\([-0.054,0.026]\)) for \textbf{DeepDTA}, and
\(-0.038\) (\([-0.082,0.006]\)) for \textbf{TAPB}. Hence, the positive TAPB
signal is localized to location and ranking rather than to residual shape
preservation after removing the mean shift. Overall, among models with
comparable AUROC, only \textbf{TAPB} exhibits a prior-selective distributional
response to the binding-pocket prior.

\begin{figure}[t]
\centering
\includegraphics[width=0.98\linewidth]{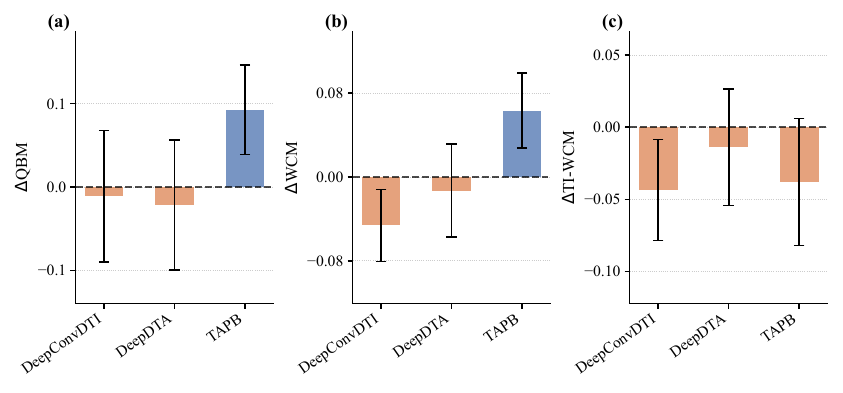}
\caption{\textsc{I-SAFE} prior-relative coherence contrasts on the Davis benchmark:
$\Delta$QBM (a), $\Delta$WCM (b), and $\Delta$TI-WCM (c), computed as spurious
minus mechanistic coherence. The dashed line marks no differential coherence;
positive values indicate greater coherence under mechanistic perturbations.
Error bars denote 95\,\% confidence intervals across five seeds.}
\label{fig:delta}
\end{figure}

\section{Discussion}
\label{sec:discuss}

\textsc{I-SAFE} introduces a distributional layer for post-hoc structural
auditing of scientific predictors. It moves the audit beyond asking whether a
model is sensitive to prior-aligned perturbations, toward asking how its raw
output distribution reorganizes under intervention. This distinction matters
because benchmark performance does not establish that model behaviour is
organized around the structures that domain knowledge identifies as relevant.

The Davis DTI case study illustrates this separation. The audited models operate
in a comparable predictive regime, yet exhibit different interventional response
profiles. \textbf{TAPB} is the only architecture with positive contrasts on both the
quantile-level and ordinal axes, while the translation-invariant component does
not provide a positive differential signal. The signal detected by
\textsc{I-SAFE} is therefore not a generic distributional effect. It is expressed
through location-level and ranking-level organization of raw output scores under
KLIFS binding-pocket perturbations. These results separate predictive
performance from distributional coherence as distinct dimensions of model
behaviour.

\paragraph{Prior-relative interpretation}
The contrast metrics are the prior-relative component of the audit. They compare the
coherence induced by perturbations of prior-selected components with the
coherence induced by controls outside the prior under the same intervention
design. A positive contrast is therefore a behavioural statement about a fixed trained
model relative to an independently specified structural prior. It does not imply
that the prior is complete or that the model has recovered the causal mechanism
of the data-generating process. It shows that the model's output distribution is
more coherently organized under perturbations of the prior-selected region.

\paragraph{Beyond scalar auditing.}
Scalar auditing~\cite{tarantino2026isaacauditingcausalreasoning} summarizes relative
interventional sensitivity as a single magnitude-based contrast. \textsc{I-SAFE}
retains more of the response structure. QBM measures coherent movement of
representative output locations, WCM measures preservation of output ordering,
and TI-WCM separates residual shape coherence from global translation. In our
experiments, \textbf{TAPB}'s positive signal appears on QBM and WCM, but not on
TI-WCM, indicating that its prior-relative behaviour is carried by location and
ranking rather than by shape preservation. This conclusion cannot be obtained
from accuracy or scalar sensitivity alone.

\paragraph{Assumptions, scope, and generalization.}
The structural prior defines the audited contrast. Here, KLIFS binding-pocket
annotations provide an external, biologically grounded prior over kinase target
residues, specifying a meaningful axis of comparison rather than a ground-truth
causal mechanism. Accordingly, the empirical claims are specific to target-side
interventions on Davis, with the drug held fixed. Outside-prior controls are
matched in operator and cardinality, but not in all local sequence or geometric
properties. Natural extensions include context-matched controls, drug-side or
joint interventions, and alternative structural priors.

Overall, \textsc{I-SAFE} contributes an evaluation methodology for scientific
AI systems whose benchmark performance is insufficient to characterize their
structural behaviour. Predictive metrics assess whether a model performs well;
scalar audits summarize whether it responds to prior-guided perturbations;
\textsc{I-SAFE} evaluates how the output distribution reorganizes under those
perturbations. By turning structurally guided interventions into interpretable distributional evidence, \textsc{I-SAFE} provides a reusable, model-agnostic evaluation protocol that supports more precise evaluative claims about black-box scientific predictors across scientific domains.


\bibliographystyle{unsrt}
\bibliography{isafe.bib}


\clearpage

\appendix

\section{Appendix}

In this appendix we report the missing proof and all the technical discussion omitted from the main body of the paper.

\subsection{Proofs}

First, we report the missing proofs.

\begin{proof}[Proof of Proposition \ref{prop0}]
    It trivially follows from the fact that

    \begin{equation}
        0\le \frac{\sum_{i=1}^K|q_i^{(\calA)}-q_i^{(\varphi(\calA))}|}{\sum_{i=1}^{|\calA|} |\fbb(x_i)-\fbb(\varphi(\pi(x_i)))|^2}.
    \end{equation}
\end{proof}

\begin{proof}[Proof of Proposition \ref{prop1}]
    Since $\pi=Id$ is a feasible point for the minimization problem
    \begin{equation}
        \label{eq:appendix_reorder}
        \min_{\pi\in \Pi_n(\calA)} |\fbb(x_i)-\fbb(\varphi(\pi(x_i)))|^2
    \end{equation}
    we infer
    \begin{equation}
        \min_{\pi\in \Pi_n(\calA)}\sum_{i=1}^{|\calA|} |\fbb(x_i)-\fbb(\varphi(\pi(x_i)))|^2 \le \sum_{i=1}^{|\calA|}|\fbb(x_i)-\fbb(\varphi(x_i))|^2,
    \end{equation}
    thus $0\le\frac{\min_{\pi\in \Pi_n(\calA)}\sum_{i=1}^{|\calA|} |\fbb(x_i)-\fbb(\varphi(\pi(x_i)))|^2}{\sum_{i=1}^{|\calA|}|\fbb(x_i)-\fbb(\varphi(x_i))|^2}\le 1$, hence the first part of the proof.
    Let us now assume that $WCM(\varphi;\fbb)=0$, we then have that
    \begin{equation}
        \min_{\pi\in \Pi_n(\calA)}\sum_{i=1}^{|\calA|} |\fbb(x_i)-\fbb(\varphi(\pi(x_i)))|^2 = \sum_{i=1}^{|\calA|}|\fbb(x_i)-\fbb(\varphi(x_i))|^2,
    \end{equation}
    meaning that the identity permutation is an optimal solution to the minimization problem in \eqref{eq:appendix_reorder}

    Let us assume that $WCM(\varphi;\fbb)=1$ and $V_\calA\neq V_{\varphi(\calA)}$.
    We then have that $\min_{\pi\in \Pi_n(\calA)}\sum_{i=1}^{|\calA|} |\fbb(x_i)-\fbb(\varphi(\pi(x_i)))|^2 =0$, meaning that there exists a permutation $\pi$ that is not equal to the identity for which it holds
    \begin{equation}
        (V_\calA)_i = (V_{\varphi(\calA)})_{\pi(i)}
    \end{equation}
    for every $i=1,\dots,|\calA|$.

    To conclude, we notice that for every $\lambda$, it holds
    \begin{equation}
        \min_{\pi\in \Pi_n(\calA)} |\lambda\fbb(x_i)-\lambda\fbb(\varphi(\pi(x_i)))|^2 = \lambda^2\min_{\pi\in \Pi_n(\calA)} |\fbb(x_i)-\fbb(\varphi(\pi(x_i)))|^2
    \end{equation}
    Likewise,
    \begin{equation}
        \sum_{i=1}^{|\calA|}|\lambda\fbb(x_i)-\lambda\fbb(\varphi(x_i))|^2 =\lambda^2\sum_{i=1}^{|\calA|}|\fbb(x_i)-\fbb(\varphi(x_i))|^2,
    \end{equation}
    hence the scalar invariant property.
\end{proof}

\begin{proof}[Proof of Proposition \ref{prop2}]
    It follows from the fact that
    \begin{equation}
        (A-D)^2+(B-C)^2 > (A-C)^2 +(B-D)^2
    \end{equation}
    whenever $A<B$ and $C<D$.
    By iteratively reordering the entries of the vectors $V_\calA$ and $V_{\varphi(\calA)}$ we are able to show that the minimum of the problem \eqref{eq:appendix_reorder} is given by the increasing reordering, thus the thesis.
\end{proof}

\subsection{Translation-Invariant WCM: Properties and Discussion}\label{app:tiwcm}

Given a probability distribution $\mu$ with finite average $m_\mu$, we denote by $\hat\mu$ as the probability distribution shifted by $-m_{\mu}$, so that the average value of $\hat\mu$ is equal to $0$.
First, we recall that, given any couple of probability distributions with finite average, namely $\mu$ and $\nu$, it holds
\begin{equation}
    W_2^2(\mu,\nu)-(m_\mu-m_\nu)^2=W_2^2(\hat\mu,\hat\nu),
\end{equation}
thus we can rewrite the Ti-WCM as follows
    \begin{equation}
        TI-WCM(\varphi,\fbb):=1-\frac{W_2(\hat\mu_{V_{\varphi(\calA)}},\hat\mu_{V_{\calA}})}{\ell_2(V_{\varphi(\calA)},V_{\calA})}.
    \end{equation}
    As a consequence, we have that $0\le TI-WCM(\varphi,\fbb)$.
    Moreover, since $(m_\mu-m_\nu)^2\ge 0$, we have
    \begin{equation}
        W_2(\hat\mu_{V_{\varphi(\calA)}},\hat\mu_{V_{\calA}})\le W_2(\mu_{V_{\varphi(\calA)}},\mu_{V_{\calA}})
    \end{equation}
    which, used in conjunction with Proposition \ref{prop1} allows us to conclude that $TI-WCM(\varphi,\fbb)\in[0,1]$.
    As a byproduct, we infer that 
    \begin{equation}
        \frac{W_2(\hat\mu_{V_{\varphi(\calA)}},\hat\mu_{V_{\calA}})}{\ell_2(V_{\varphi(\calA)},V_{\calA})}\le \frac{W_2(\mu_{V_{\varphi(\calA)}},\mu_{V_{\calA}})}{\ell_2(V_{\varphi(\calA)},V_{\calA})}
    \end{equation}
    meaning that $0\le WCM(\varphi;\fbb)\le TI-WCM(\varphi;\fbb)\le 1$ for every $\varphi$ and every black-box model $\fbb$.
    In particular, whenever $WCM(\varphi;\fbb)=1$ then $TI-WCM(\varphi;\fbb)=1$ likewise, if $TI-WCM(\varphi;\fbb)=0$ then $WCM(\varphi;\fbb)=0$.
    We then notice that the TI-WCM is a stricter metric than the WCM.
    Indeed, if $TI-WCM(\varphi;\fbb)=1$ it must be the case that 
    \begin{enumerate}[label=(\roman*)]
        \item the model output distribution after the perturbation has the same average value as the model output distribution before the permutation and 
        \item it induces the same ranking ordering. 
    \end{enumerate}

    Lastly, it is easy to adapt the argument used to prove Proposition \ref{prop1} to show that also TI-WCM is scale invariant.

\subsection{Additional Tables}\label{app:tab}

\begin{table}[h]
\caption{Structural coverage and auditing-set composition for the Davis
benchmark. Rows report the successive requirements needed to define
well-posed KLIFS-based interventions under the matched cardinality design.}
\label{app:tab:coverage}
\centering
\small
\begin{tabular}{lcc}
\toprule
Criterion & Count & Fraction of test split \\
\midrule
Targets in test split           & 379   & \\
With KLIFS annotation           & 321   & 84.7\% \\
With realizable interventions   & 208   & 54.9\% \\
\midrule
Drug--target pairs in \(\calA\) & 3,044 & \\
\quad of which \(y=1\)          & 154   & 5.1\% \\
\quad of which \(y=0\)          & 2,890 & 94.9\% \\
\midrule
Median pocket size \(|\calP(t)|\) & 85 residues & \\
Exact cardinality matching        & 100\% & \\
\bottomrule
\end{tabular}
\end{table}

\begin{table}[h]
\caption{Predictive performance on the Davis auditing subset
\((208\) targets, \(3{,}044\) drug--target pairs). AUROC is reported as mean
with 95\% confidence intervals across five training seeds and is used only to
verify a comparable predictive regime before interventional auditing.}
\label{app:tab:auroc}
\centering
\small
\begin{tabular}{lc}
\toprule
Model & AUROC (95\% CI) \\
\midrule
DeepConvDTI~\cite{lee2019deepconv} & 0.876 (0.875, 0.878) \\
TAPB~\cite{lin2025tapb}            & 0.882 (0.851, 0.899) \\
DeepDTA~\cite{ozturk2018deepdta}   & 0.907 (0.902, 0.913) \\
\bottomrule
\end{tabular}
\end{table}

\begin{table}[h]
\caption{\textsc{I-SAFE} prior-relative coherence contrasts on the Davis benchmark (operator=all, five seeds, 95\,\% CI). Contrasts are defined as spurious minus mechanistic. Positive values indicate greater coherence under mechanistic
perturbations than under spurious controls.}
\label{app:tab:delta}
\centering
\small
\begin{tabular}{lccc}
\toprule
Model
  & $\Delta$QBM (95\,\% CI)
  & $\Delta$WCM (95\,\% CI)
  & $\Delta$TI-WCM (95\,\% CI) \\
\midrule
DeepConvDTI
  & $-$0.011 ($-$0.090,\phantom{$-$}0.068)
  & $-$0.046 ($-$0.081, $-$0.012)
  & $-$0.043 ($-$0.078, $-$0.009) \\
DeepDTA
  & $-$0.021 ($-$0.099,\phantom{$-$}0.057)
  & $-$0.013 ($-$0.057,\phantom{$-$}0.031)
  & $-$0.014 ($-$0.054,\phantom{$-$}0.026) \\
TAPB
  & $+$0.093 (\phantom{$-$}0.039,\phantom{$-$}0.147)
  & $+$0.063 (\phantom{$-$}0.028,\phantom{$-$}0.099)
  & $-$0.038 ($-$0.082,\phantom{$-$}0.006) \\
\bottomrule
\end{tabular}
\end{table}

\begin{table}[h]
\caption{QBM sensitivity to the quantile grid (operator=all,
five seeds, 95\,\% CI).
$\Delta\text{QBM} = \text{QBM}_{\text{spur}} - \text{QBM}_{\text{mech}}$;
positive values indicate greater coherence under mechanistic perturbations.}
\label{app:tab:qbm-sensitivity}
\centering
\small
\begin{tabular}{llccc}
\toprule
Grid & Model
  & QBM$_{\text{mech}}$ (95\,\% CI)
  & QBM$_{\text{spur}}$ (95\,\% CI)
  & $\Delta$QBM (95\,\% CI) \\
\midrule
\multirow{3}{*}{$K=3$}
  & DeepConvDTI & 0.737 (0.678, 0.799) & 0.726 (0.673, 0.774) & $-$0.011 ($-$0.090,\phantom{$-$}0.068) \\
  & DeepDTA     & 0.842 (0.784, 0.897) & 0.821 (0.766, 0.873) & $-$0.021 ($-$0.099,\phantom{$-$}0.057) \\
  & TAPB        & 0.515 (0.476, 0.553) & 0.607 (0.572, 0.647) & $\mathbf{+}$0.093 (\phantom{$-$}0.039,\phantom{$-$}0.147) \\
\midrule
\multirow{3}{*}{$K=5$}
  & DeepConvDTI & 0.547 (0.486, 0.608) & 0.509 (0.464, 0.556) & $-$0.038 ($-$0.114,\phantom{$-$}0.038) \\
  & DeepDTA     & 0.779 (0.710, 0.846) & 0.746 (0.687, 0.807) & $-$0.033 ($-$0.123,\phantom{$-$}0.058) \\
  & TAPB        & 0.466 (0.419, 0.507) & 0.519 (0.474, 0.561) & $+$0.054 ($-$0.008,\phantom{$-$}0.116) \\
\midrule
\multirow{3}{*}{$K=9$}
  & DeepConvDTI & 0.617 (0.571, 0.660) & 0.586 (0.552, 0.622) & $-$0.031 ($-$0.088,\phantom{$-$}0.026) \\
  & DeepDTA     & 0.799 (0.752, 0.846) & 0.771 (0.723, 0.815) & $-$0.028 ($-$0.093,\phantom{$-$}0.038) \\
  & TAPB        & 0.484 (0.456, 0.513) & 0.553 (0.519, 0.584) & $\mathbf{+}$0.069 (\phantom{$-$}0.026,\phantom{$-$}0.112) \\
\bottomrule
\end{tabular}
\end{table}

\clearpage
\subsection{Computational Details}
\label{app:compute}

All \textsc{I-SAFE} analyses are performed post hoc on fixed model checkpoints,
with no retraining, fine-tuning, or audit-specific model optimization. To support
reproducibility, the supplemental archive includes the per-seed audit outputs
from which all reported results are derived. Starting from these outputs, the
reproduction pipeline, comprising AUROC verification, \textsc{I-SAFE} metric
estimation, confidence-interval computation, figure generation, and QBM
sensitivity analysis, completes in approximately 17 minutes and does not require
dedicated GPU acceleration. This runtime was measured on a Windows 11 laptop
with an AMD Ryzen AI 7 350 processor, 31.3 GB RAM, and Python 3.10.19 from
conda-forge. Model training is separate from the released audit pipeline and
followed the protocols and original implementations described in prior
work~\cite{ozturk2018deepdta,lee2019deepconv,lin2025tapb}.


\end{document}